\pdfoutput=1
\documentclass[letterpaper, 10 pt, conference]{cls/ieeeconf}

\IEEEoverridecommandlockouts  
\let\labelindent\relax

\usepackage[table,usenames,dvipsnames]{xcolor}
\usepackage{cite}
\usepackage{amsmath,amssymb,amsfonts,amsthm,amscd,dsfont}
\usepackage{accents} 
\usepackage{bbm}
\usepackage{algorithm,listings}
\usepackage[noend]{algpseudocode}
\usepackage{graphicx,tabularx,adjustbox}
\usepackage[font={small}]{caption}
\usepackage{subcaption}
\usepackage{paralist}

\let\appendices\relax
\usepackage[titletoc,title]{appendix}
\usepackage{balance}
\usepackage[breaklinks=true, colorlinks, bookmarks=true, citecolor=Black, urlcolor=Violet,linkcolor=Black]{hyperref}

\newcommand{\calB}{{\cal B}}

\newcommand{\calE}{{\cal E}}
\newcommand{\calF}{{\cal F}}
\newcommand{\calG}{{\cal G}}
\newcommand{\calH}{{\cal H}}
\newcommand{\calI}{{\cal I}}
\newcommand{\calJ}{{\cal J}}

\newcommand{\calN}{{\cal N}}

\newcommand{\calP}{{\cal P}}

\newcommand{\calR}{{\cal R}}

\newcommand{\calV}{{\cal V}}

\newcommand{\bfb}{\mathbf{b}}

\newcommand{\bfd}{\mathbf{d}}

\newcommand{\bfm}{\mathbf{m}}

\newcommand{\bfp}{\mathbf{p}}

\newcommand{\bfv}{\mathbf{v}}

\newcommand{\bfz}{\mathbf{z}}

\newcommand{\bftheta}{\boldsymbol{\theta}}

\newcommand{\bfrho}{\boldsymbol{\rho}}

\newcommand{\bfpsi}{\boldsymbol{\psi}}

\newcommand{\bfxi}{\boldsymbol{\xi}}

\newcommand{\bfR}{\mathbf{R}}

\newcommand{\bfU}{\mathbf{U}}
\newcommand{\bfV}{\mathbf{V}}

\newcommand{\bfX}{\mathbf{X}}

\newcommand{\bbR}{\mathbb{R}}

\newcommand{\bbU}{\mathbb{U}}

\newcommand{\crl}[1]{\left\{#1\right\}}
\newcommand{\ubar}[1]{\underaccent{\bar}{#1}}

\newtheorem{proposition}{Proposition}

\theoremstyle{definition}
\newtheorem{definition}{Definition}
\newtheorem{condition}{Condition}

\newtheorem*{problem}{Problem}
\theoremstyle{remark}

\title{\LARGE \bf Active Mapping via Gradient Ascent Optimization of Shannon Mutual Information over Continuous \textit{SE(3)} Trajectories%
\thanks{}%
}

\author{Arash Asgharivaskasi \and Shumon Koga \and Nikolay Atanasov
\thanks{We gratefully acknowledge support from NSF FRR CAREER 2045945 and ARL DCIST CRA W911NF-17-2-0181.}%
\thanks{The authors are with the Department of Electrical and Computer Engineering, University of California San Diego, CA 92093, USA {\tt\small \{aasghari,skoga,natanasov\}@eng.ucsd.edu}.}
}

\begin{document}

\maketitle

\begin{abstract}
The problem of active mapping aims to plan an informative sequence of sensing views given a limited budget such as distance traveled. This paper consider active occupancy grid mapping using a range sensor, such as LiDAR or depth camera. State-of-the-art methods optimize information-theoretic measures relating the occupancy grid probabilities with the range sensor measurements. The non-smooth nature of ray-tracing within a grid representation makes the objective function non-differentiable, forcing existing methods to search over a discrete space of candidate trajectories. This work proposes a differentiable approximation of the Shannon mutual information between a grid map and ray-based observations that enables gradient ascent optimization in the continuous space of \textit{SE(3)} sensor poses. Our gradient-based formulation leads to more informative sensing trajectories, while avoiding occlusions and collisions. The proposed method is demonstrated in simulated and real-world experiments in 2-D and 3-D environments.
\end{abstract}


\section{Introduction}
\label{sec:introduction}

Mapping an unknown environment using sensor-equipped mobile robots has been widely studied motivated by many real-world applications such as search and rescue operations, planetary exploration, security and surveillance. While the traditional mapping methods have been developed for a given robot trajectory, optimizing the path for increasing the accuracy of the constructed map is significant for the operation of autonomous robots in highly unstructured environment. Such an ``active mapping" problem renders challenges in computation of objective function and its reliable optimization method in terms of enhancing the performance and computational feasibility. 


\begin{figure}[t]
    \includegraphics[width=0.95\linewidth]{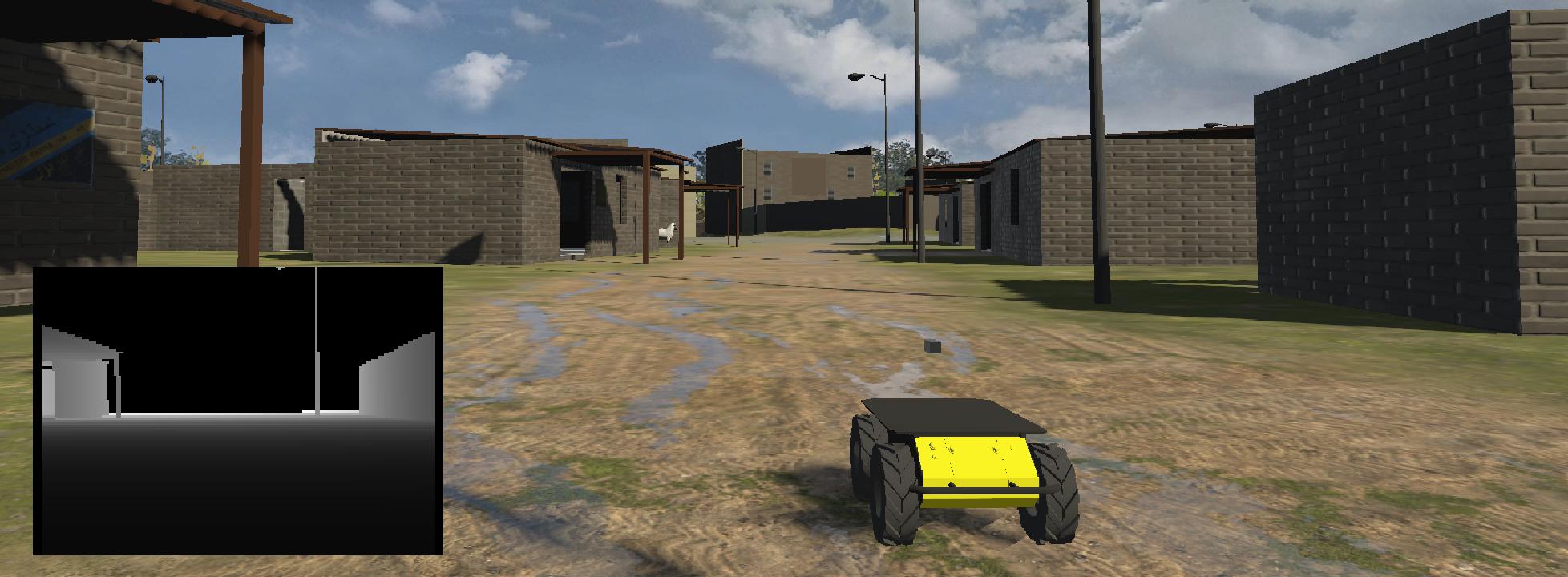}\\[1ex]
    \includegraphics[width=0.95\linewidth]{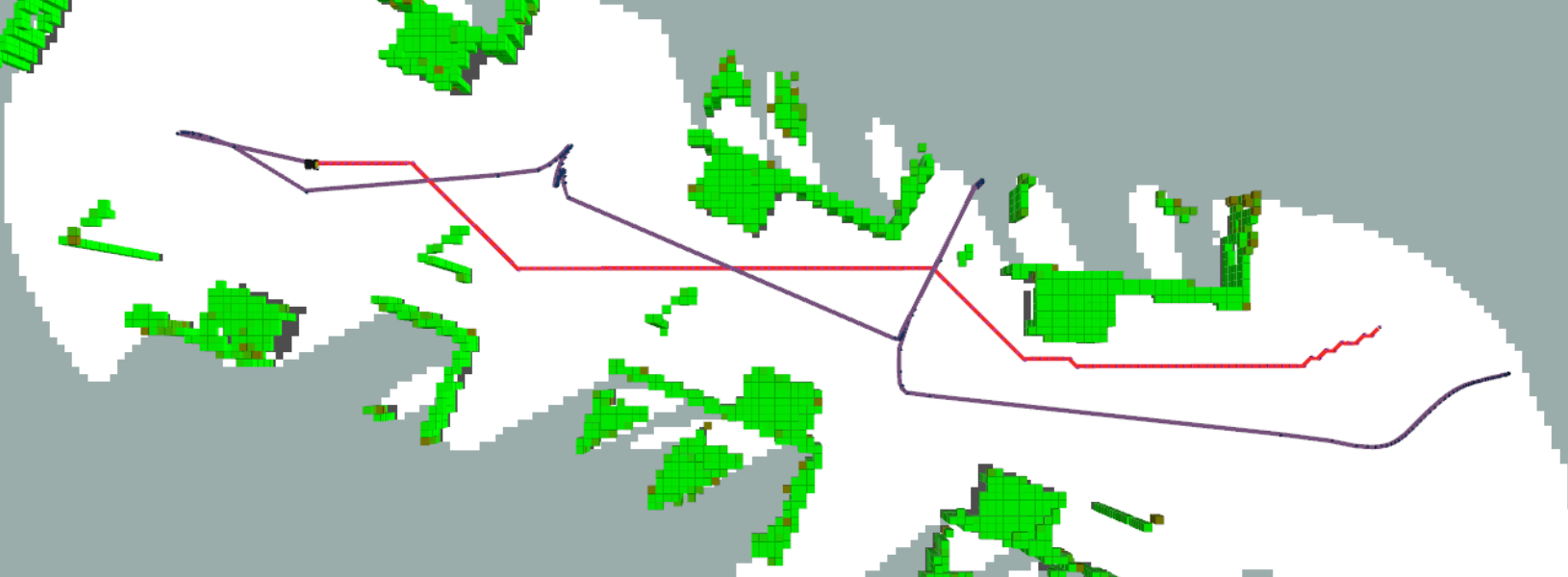}
    %
    \caption{Occlusion and collision-aware active mapping. Top: Wheeled robot exploring an unknown environment equipped with a depth sensor used for estimating a 3-D map. Bottom: Path planning given the current map. The red trajectory is $A^*$ planning towards a frontier. The purple trajectory is obtained via gradient ascent optimization of the objective function.}
    \label{fig:intro}
\end{figure}

One of the most common procedures for incremental estimation of occupancy grid maps is to utilize a narrow-ray range sensor, such as Laser range finder
\cite[Ch.~9]{ProbabilisticRoboticsBook}. However, the resulting observation model occurs to be a piece-wise constant, hence non-smooth, function of measurement poses due to the discrete nature of ray-tracing over a grid map \cite[Ch.~5]{julian2013mutual}. Hence, the objective function is piecewise-constant; making the derivatives either $0$ or undefined.
In addition to the problem of non-smoothness of grid map representations, occlusion, which is caused when a field of view intersects with objects, is a significant challenge in sensing and planning to accomplish several tasks in robotics such as 3-D reconstruction \cite{maver1993occlusions}, object classification \cite{atanasov2014nonmyopic}, and target tracking \cite{ma2021fov}. Koga \textit{et al.}~\cite{icr} has developed dense active mapping using iterative Covariance Regulation (iCR), which utilizes a differentiable field of view that enables deriving an explicit gradient of an information-theoretic objective with respect to a multi-step control sequence. Incorporating occlusion in iCR makes the field of view dependent on the map estimate, which renders significant challenge for sensing model formulation since it violates the linear-Gaussian property of the mapping by Kalman Filter. Nevertheless, the current work borrows one of the key ideas from iCR: In order to perform gradient-based optimization of the mutual information, one needs to apply a smooth proximity operator among the elements involved in its computation.

The work by Rocha \textit{et al.}~\cite{rocha2005cooperative} was among the first that attempted to derive a gradient-based strategy for active mapping of occupancy grids. The authors proposed a gradient of map entropy with respect to the robot pose at a cell center via finite difference of entropy values at adjacent cells. Julian \cite{julian2013mutual} proposed a divergent beam sensor model, where the width of a beam increases radially as it travels farther through space. While the derived mutual information formula was shown to be differentiable, it suffers from high computational complexity as it requires numerical integration of the objective function. Charrow \textit{et al.}~\cite{charrow2015information} proposed a numerical evaluation for the gradient of Cauchy-Schwarz mutual information (CSQMI)~\cite{csqmi} using finite differences of CSQMI evaluated at cell centers. Our work is most similar to \cite{DiAIM} and \cite{DiAIM-3D}, where the authors formulate the information gain as a sum of informative elements weighted by a discount factor. In particular, \cite{DiAIM} defines informative elements as frontier cells between free and unexplored area visible from a candidate pose. However, unlike the mutual information between the map and a sensor observation, using visible frontier size as a proxy for information gain does not take into account the effect of sensor noise which is inevitable in real-world sensing applications \cite{csqmi, frontier-3d-failure}.

The present work distinguishes itself from the prior methods by proposing a active mapping strategy that allows gradient ascent optimization of the Shannon mutual information (SMI) between the grid map and a sequence of beam-based observations. As opposed to the discrete-space active mapping methods \cite{frontier, csqmi, fsmi, ssmi_icra} that aim to plan an informative robot path through evaluating a finite set of candidate trajectories, the current work finds an optimal trajectory over the continuous space of the robot state. Fig.~\ref{fig:intro} illustrates the proposed gradient-based active mapping using a depth sensor. Our main \textbf{contributions} are:
\begin{enumerate}
  \item a differentiable interpolation of the SMI as well as a closed-form gradient expression,
  \item decomposition of the SMI into additive terms over a robot trajectory in \textit{SE(3)}, under sufficient assumptions for the interpolation method.
\end{enumerate}
Unlike the previous works that used map entropy or visible number of frontiers, directly utilizing the SMI allows incorporating the range sensor noise specifications to the objective function; leading to more accurate exploration.
Moreover, gradient-based methods allow augmenting the objective function with other differentiable terms (e.g. localization accuracy \cite{knuth2012collaborative}, path cost \cite{mizuno2019enhanced}). With the addition of a differentiable collision penalty to the objective function, we propose an occlusion and collision-aware robot exploration.

\section{Problem Statement}
\label{sec:problem_statement}

Consider a robot with pose $\bfX_t \in \textit{SE(3)}$ at time $t$:
\begin{equation}
    \bfX_t := \begin{bmatrix} \bfR_t & \bfp_t \\ \mathbf{0}^\top & 1 \end{bmatrix},
\label{eq:dynamic_model}
\end{equation}
where $\bfR_t \in SO(3)$ is the robot orientation and $\bfp_t \in \mathbb{R}^3$ is the robot position. The robot is navigating in an environment composed of occupied and free space. A mounted range sensor, e.g. LiDAR or depth camera, provides the robot with a stream of beam-based observations $\bfz_t \in \mathbb{R}^B$, where $B$ is the number of beams in a laser scan or pixels in a depth image, measuring the distance from the robot's position to the closest obstacle along the beam.
We model the map $\bfm$ as a grid of cells $m_i$, $i \in \calI := \crl{1, \ldots, N}$, where each cell can take one of the two states: \textit{free} or \textit{occupied}. To model measurement noise, we consider a probability density function (PDF) $p(\bfz_t \mid \bfm, \bfX_t)$ for each observation. Let $p_t(\bfm) = p(\bfm \mid \calH_t)$ be the probability mass function (PMF) of the map $\bfm$ given the history of robot poses and observations $\calH_t = \crl{(\bfX_\tau, z_\tau)}^t_{\tau=1}$. A new observation $\bfz_{t+1}$ made from robot pose $\bfX_{t+1}$ can then be integrated into the map estimation process using Bayes rule:
\begin{equation}
    p_{t+1}(\bfm) \propto p(\bfz_{t+1} \mid \bfm, \bfX_{t+1}) p_t(\bfm).
\label{eq:bayes_rule}
\end{equation}
%


The goal is to choose a collision-free pose trajectory to obtain maximally informative measurements for constructing an accurate map. As shown by Julian \textit{et al.}~\cite{julian}, maximizing the Shannon Mutual Information (SMI) between the map $\bfm$ and a sequence of potential future measurements $\bfz_{t+1:t+T}$ yields an efficient active mapping strategy. The SMI is defined as:
\begin{equation}
\begin{aligned}
    &I(\bfm; \bfz_{t+1:t+T} \mid \bfX_{t+1:t+T}, \calH_t) :=\\
    &\sum_{\bfm \in 2^N} \!\int \cdots \int p(\bfm, \bfz_{t+1:t+T} \mid \bfX_{t+1:t+T}, \calH_t) \times\\
    &\log\frac{p(\bfm, \bfz_{t+1:t+T} \mid \bfX_{t+1:t+T}, \calH_t)}{p(\bfm \mid \calH_t) p(\bfz_{t+1:t+T} \mid \bfX_{t+1:t+T}, \calH_t)} \prod_{\tau=1}^T d\bfz_{t+\tau},
\end{aligned}
\label{eq:mut_info_generic}
\end{equation}
where $\calH_t$ represents the realized history of robot poses and observations and, hence, does not appear as an integration variable. Throughout this paper, we assume that the robot pose $\bfX_t$ is known for all $t$. The SMI is a function $I(.): \textit{SE(3)}^T \rightarrow \mathbb{R}_{\geq 0}$ of the robot trajectory $\bfX_{t+1:t+T}$ parameterized by $\calH_t$.

\begin{problem}
Given a map PMF $p_t(\bfm)$ obtained from prior robot poses and observations $\calH_t$ and a finite planning horizon $T$, find a pose trajectory $\bfX_{t+1:t+T} \in \textit{SE(3)}^T$ that maximizes the SMI between the map $\bfm$ and the future observations $\bfz_{t+1:t+T}$ with PDF in \eqref{eq:bayes_rule}:
\begin{align}
    \max_{\bfX_{t+1:t+T} \in \textit{SE(3)}^T}  \Big(  I(\bfm; \bfz_{t+1:t+T} \mid &\bfX_{t+1:t+T}, \calH_t) - \label{eq:reward_func}\\
    &\qquad \quad \gamma_c C(\bfX_{t+1:t+T}) \Big),\notag
\end{align}
where $C(\bfX_{t+1:t+T})$ is a penalty term capturing the cost of collisions along $\bfX_{t+1:t+T}$ and $\gamma_c \geq 0$ is the weight of the collision penalty.
\end{problem}

In the next section, we propose a differentiable approximation of the SMI function that can be utilized for gradient-based optimization of \eqref{eq:reward_func}.

\section{Proposed Method}
\label{sec:proposed_method}

\subsection{Notation}

We overload $\hat{(.)}$ to denote the mapping from an axis-angle vector $\bftheta \in \mathbb{R}^3$ to a $3 \times 3$ skew-symmetric matrix $\hat{\bftheta} \in \mathfrak{so}(3)$ as well as from a vector $\bfxi \in \mathbb{R}^6$ to a $4 \times 4$ twist matrix:
\begin{equation}
\bfxi = \begin{bmatrix} \bfrho\\\bftheta \end{bmatrix} \in \bbR^6 \qquad \hat{\bfxi} := \begin{bmatrix} \hat{\bftheta} & \bfrho \\ \mathbf{0}^\top & 0 \end{bmatrix} \in \mathfrak{se}(3).
\end{equation}
We define an infinitesimal change of pose $\bfX \in SE(3)$ using a right perturbation $\bfX\exp(\hat{\bfxi}) \in SE(3)$. The functions $\log(.): \textit{SE(3)} \rightarrow \mathfrak{se}(3)$ and $(.)^\vee: \mathfrak{se}(3) \rightarrow \bbR^6$ denote the inverse mappings associated with $\exp(.)$ and $\hat{(.)}$, respectively. Please refer to \cite[Ch.7]{BarfootBook} for details.


\subsection{One Step Ahead Planning}
\label{subsec:one_step_ahead}

We first study the case where $T = 1$, i.e., the robot view is optimized only one step into the future. The core idea is to introduce the notion of viewpoint grid, which is a discrete set of candidate robot poses. Then, the SMI with respect to an arbitrary robot pose is approximated as a linear combination of the SMI with respect to the candidate poses with a differentiable function with respect to an arbitrary robot pose. 
%
\begin{definition}
A viewpoint grid $\calG$ is a set of robot poses $\bfX \in \textit{SE(3)}$ with position $\bfp \in \bbR^3$ and orientation $\bfR \in SO(3)$ such that $(\bfp, \bfR) \in \calP \times \calR$, where $\calP$ is the set of all map cell centers and $\calR$ is a finite set of orientations.
\end{definition}
We approximate the SMI at pose $\bfX$ as a convex combination of the SMI computed over all poses in the view grid $\calG$. Namely, $I(\bfm; \bfz \mid \bfX, \calH_t) \approx \Tilde{I}(\bfm; \bfz \mid \bfX, \calH_t) $, and 
\begin{equation}
\begin{aligned}
    &\Tilde{I}(\bfm; \bfz \mid \bfX, \calH_t) := \sum_{\bfV \in \calG} \alpha_\bfV(\bfX) I(\bfm; \bfz \mid \bfV, \calH_t), \\
    &\text{subject to}\;\; \sum_{\bfV \in \calG} \alpha_\bfV(\bfX) = 1,\; \forall \bfV \in \calG:\; 0 \leq \alpha_\bfV(\bfX),
\end{aligned}
\label{eq:SMI_approx_single}
\end{equation}
where $\alpha_\bfV(\bfX)$ is a differentiable function with respect to robot pose $\bfX$. In practice, evaluating the SMI for all $\bfV \in \calG$ during each planning phase would be computationally expensive. Therefore, one needs to design a distance metric for $\alpha_\bfV(\bfX)$ that is only non-zero in a close vicinity of the robot pose $\bfX$. The idea of pulling the robot pose $\bfX$ out of the SMI function $I(.)$ using a differentiable weighting function $\alpha_\bfV(\bfX)$ makes it possible to obtain non-zero derivatives for the approximate SMI function $\Tilde{I}(.)$ with respect to the robot pose. Fig.~\ref{fig:smi_approx} illustrates the SMI approximated from $3$ viewpoints; note how each viewpoint creates a `field' of information in its vicinity, while the approximate SMI equals the net influence of all $3$ viewpoints.

The differentiable property of the approximate SMI enables gradient-based optimization of the robot pose in order to generate maximally informative observations. This is done via applying a small perturbation $\bfpsi \in \bbR^6$ in the robot frame to the pose $\bfX$ along the direction of the gradient with a step size of $l$:
%
\begin{align}
    \bfX^{(k+1)} & = \bfX^{(k)} \exp(l \hat{\bfpsi}^{(k)}), \notag\\
    \bfpsi^{(k)} & =\nabla_{\bfpsi} \Tilde{I}(\bfm; \bfz \mid \bfX^{(k)}\exp(\hat{\bfpsi}), \calH_t) \vert_{\bfpsi = 0}  \label{eq:grad_ascent}\\
    & = \sum_{\bfV \in \calG} \nabla_{\bfpsi} \alpha_\bfV(\bfX^{(k)} \exp (\hat{\bfpsi} )) |_{\bfpsi = 0} I(\bfm; \bfz \mid \bfV, \calH_t). \notag
\end{align}
Note that the gradient is a $6$-dimensional vector since the robot pose in \textit{SE(3)} has $6$ degrees of freedom. In Sec.~\ref{subsec:occ_coll_active_mapping}, we derive a closed form expression of the gradient for a particular selection of the weighting function $\alpha_{\bfV}(\bfX)$. The gradient ascent rule of \eqref{eq:grad_ascent} concludes our method proposition for finding the (locally) most informative next robot pose when the planning horizon $T$ is set to $1$. In the following, we discuss the case where we are interested in optimizing a multi-step sequence of robot poses, given a history of prior observations $\calH_t$.

\begin{figure}[t]
    \centering
    \includegraphics[width=0.95\linewidth]{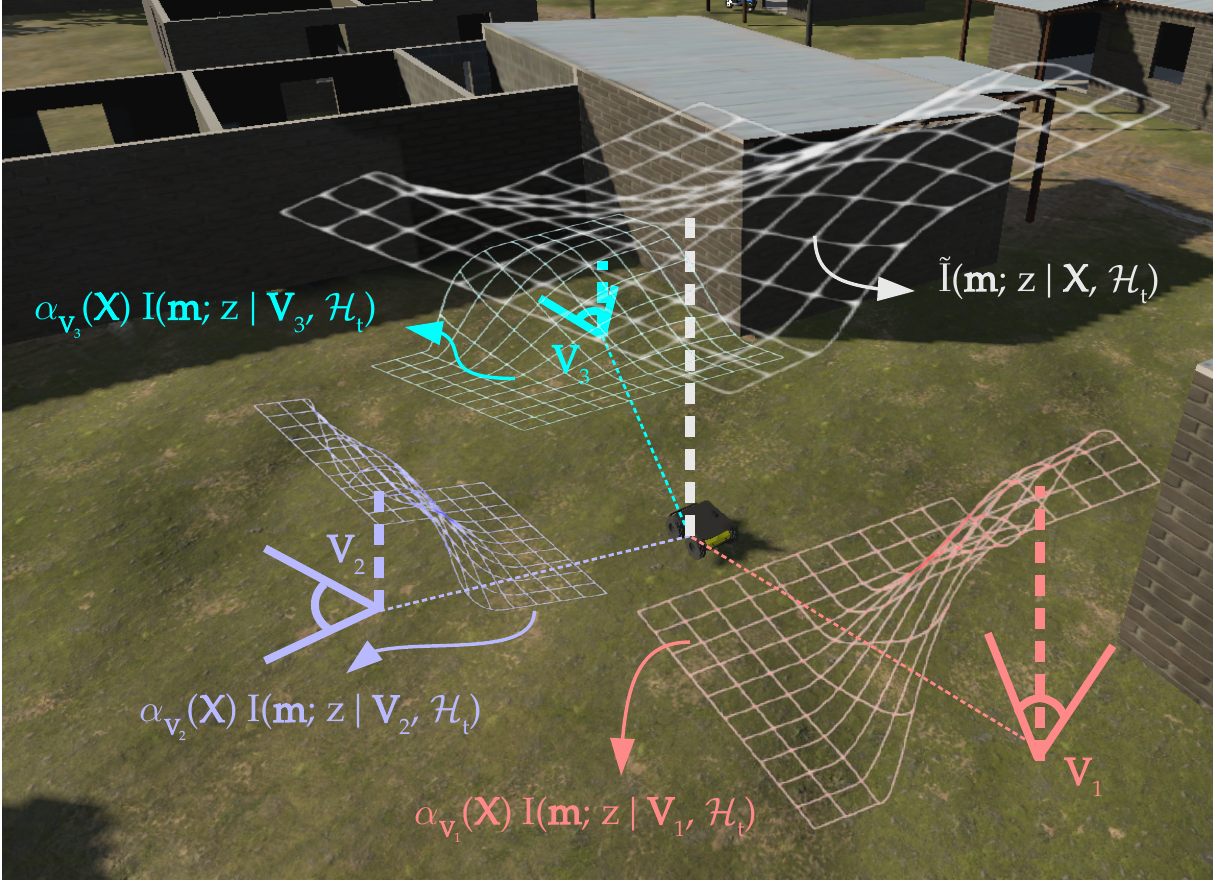}
    \caption{The SMI approximation via $3$ viewpoints, colored differently for each $\bfV_i$, $i \in \crl{1, 2, 3}$. The field of view of each viewpoint determines their corresponding SMI $I(\bfm, \bfz \mid \bfV_i, \calH_t)$, while the weight $\alpha_{\bfV_i}(\bfX)$ dictates the contribution of $\bfV_i$ to the approximated SMI $\Tilde{I}(\bfm, \bfz \mid \bfX, \calH_t)$, colored white.}
    \label{fig:smi_approx}
\end{figure}

\subsection{Trajectory Optimization for Active Mapping}
\label{subsec:traj_optim}

Here, we present the approximation of the SMI for a trajectory, rather than a single robot pose. Analogous to \eqref{eq:SMI_approx_single}, we define the approximated SMI for robot trajectory $\bfX_{t+1:t+T}$ as a convex combination over all trajectories in the set $\calG^T$. Namely, $ I(\bfm; \bfz_{t+1:t+T} \mid  \bfX_{t+1:t+T}, \calH_t) \approx  \Tilde{I}(\bfm; \bfz_{t+1:t+T} \mid \bfX_{t+1:t+T}, \calH_t) ,$ and 
\begin{equation}
\begin{aligned}
    \Tilde{I}(\bfm; \bfz_{t+1:t+T} &\mid \bfX_{t+1:t+T}, \calH_t) :=\\
    \sum_{\calV \in \calG^T} &A_{\calV}(\bfX_{t+1:t+T}) I(\bfm; \bfz_{t+1:t+T} \mid \calV, \calH_t),
\end{aligned}
\label{eq:SMI_approx_multi}
\end{equation}
where the weighting function $A_{\calV}(\bfX_{t+1:t+T})$ for a trajectory $\calV = V_{t+1:t+T}$ is defined as:
\begin{equation}
    A_{\calV}(\bfX_{t+1:t+T}) := \prod_{\tau=1}^T \alpha_{\bfV_{t+\tau}}(\bfX_{t+\tau}).
\end{equation}
The terms $\alpha_{\bfV_{t+\tau}}(\bfX_{t+\tau})$ follow the same properties as in \eqref{eq:SMI_approx_single}, making \eqref{eq:SMI_approx_multi} a convex combination of the SMI terms. Computing $I(\bfm; \bfz_{t+1:t+T} \mid \calV, \calH_t)$ requires integration over all instances of combined observations $\bfz_{t+1:t+T}$ and should be repeated for all $\calV \in \calG^T$, which is computationally infeasible during the planning time. We aim to impose a structure on the weighting function $\alpha_{\bfV_{t+\tau}}(\bfX_{t+\tau})$ that allows breaking down $I(\bfm; \bfz_{t+1:t+T} \mid \calV, \calH_t)$ into independent additive terms. The following conditions enable such a decomposition:
\begin{condition}
Given a robot pose $\bfX$, $\alpha_\bfV(\bfX)$ is non-zero only for a subset $\bar{\calG}(\bfX) \subset \calG$ of viewpoints within a distance $\xi_{{\rm max}}$ from $\bfX$.
\label{cond:1}
\end{condition}
\begin{condition}
Let $\calF \subset \bbR^3$ be the unobstructed field of view (FOV) of the sensor in robot frame with homogeneous representation $\ubar{\calF}$. Also, let $\bbU(\bfX) := \cup_{\bfV \in \bar{\calG}(\bfX)} \bfV \ubar{\calF} \subset \bbR^3$ be the union of all FOVs belonging to $\bar{\calG}(\bfX)$ in the world frame. This condition states that, for any pair of robot poses $\bfX_{i}$ and $\bfX_{j}$ ($i \neq j$) in trajectory $\bfX_{t+1:t+T}$, we have $\bbU(\bfX_{i}) \cap \bbU(\bfX_{j}) = \emptyset$.
\label{cond:2}
\end{condition}
The above conditions are sufficient for decomposing the approximate SMI of trajectory $\bfX_{t+1:t+T}$ to $T$ independent additive terms, resulting in a computationally feasible trajectory optimization formula. The main idea comes from the fact that, given the above conditions, observations $\bfz_{i}$ and $\bfz_{j}$ ($i \neq j$) made from viewpoints inside $\calG(\bfX_{i})$ and $\calG(\bfX_{j})$, respectively, are independent random variables. Please refer to the proof of the proposition below for more details.
\begin{proposition}
\label{prop:approx_smi_multi}
Under Cond.~\ref{cond:1} and \ref{cond:2}, the approximated SMI can be expressed as the sum of individual SMI approximations for each pose in the trajectory $\bfX_{t+1:t+T}$:
\begin{equation}
\begin{aligned}
    \Tilde{I}(\bfm; \bfz_{t+1:t+T} &\mid \bfX_{t+1:t+T}, \calH_t) =\\
    &\sum_{\tau=1}^T \Tilde{I}(\bfm; \bfz_{t+\tau} \mid \bfX_{t+\tau}, \calH_t).
\end{aligned}
\label{eq:approx_smi_multi}
\end{equation}
\end{proposition}
\begin{proof}
See Appendix~\ref{app:approx_smi_multi}.
\end{proof}
The result of Prop.~\ref{prop:approx_smi_multi} enables computationally feasible trajectory optimization for robot exploration in an unknown environment. Since each term in \eqref{eq:approx_smi_multi} is only dependent upon a single pose in the trajectory, the gradient ascent rule in \eqref{eq:grad_ascent} can be directly employed to update each robot pose $\bfX_{t + \tau}$, $\tau \in \crl{1, \ldots, T}$. In the following part, we introduce a practical gradient-based solution to the problem of active mapping stated in \eqref{eq:reward_func}.

\subsection{Active Mapping via Gradient Ascent}
\label{subsec:occ_coll_active_mapping}

A key advantage of gradient-based optimization is the possibility of adding various reward or penalty terms to the objective function, enabling achievement of a more complex optimization goal. We begin by defining the collision penalty term $C(\bfX_{t+1:t+T})$ in the objective function of \eqref{eq:reward_func}, which is responsible for driving the optimized robot pose away from obstacles within the environment. However, since we do not know the map in prior, we resort to the estimation of the map to extract the obstacles.
%
\begin{definition}
\label{def:free_dist}
Let $\hat{\calE}_f(p_t(\bfm))$ be the maximum-likelihood estimation of the free space at time $t$. For a position $\bfp \in \bbR^3$, we define \textit{free distance} as follows:
\begin{equation}
    d(\bfp, p_t(\bfm)) = \min_{\bfb \in \partial\hat{\calE}_f(p_t(\bfm))} \|\bfp - \bfb\|_2.
\end{equation}
\end{definition}
It is important to consider that large mutual information occurs near the boundary between the free space and the unknown parts of the map \cite{julian}. Therefore, one should seek a balance between large clearance from obstacles and informativeness of observation made from the resulting robot pose. We define the collision cost $C(\bfX_{t+1:t+T})$ as sum of the log-values of inverse free distance for each pose $\bfX_{t+\tau}$ with position $\bfp_{t+\tau}$ in the trajectory:
\begin{equation}
    C(\bfX_{t+1:t+T}) = - \sum_{\tau = 1}^T \log{(d(\bfp_{t+\tau}, p_t(\bfm)))}.
\end{equation}
Using a logarithmic scale causes large penalty for poses close to obstacles, while it does not discourage approaching the unknown region from a safe distance due to its suppressed gradient over large inputs.

In addition to the collision cost, we add a penalty term to the objective function to enforce Cond.~\eqref{cond:2} during each planning phase, minimizing the overlap among the sensor FOVs $\bfX_{t+\tau} \ubar{\calF}$ in the candidate trajectory. We consider a pair-wise penalty term for poses within the trajectory as follows:
%
\begin{equation}
    \begin{aligned}
    q(\bfp_i, \bfp_j) &= \max\crl{0, 2 \delta_{q}-\|\bfp_i - \bfp_j\|_2}^2\\
    \delta_{q} &= |\calF| + \xi_{{\rm max}}, 
\end{aligned}
\label{eq:overlap_penalty}
\end{equation}
where $\bfp_i$ and $\bfp_j$ are robot positions for poses $\bfX_i$ and $\bfX_j$, respectively, $|\calF|$ is the diameter of $\calF$, and $\xi_{{\rm max}}$ denotes the maximum distance from robot pose $\bfX$ to a viewpoint in $\bar{\calG}(\bfX)$ (cf. Cond.~\ref{cond:1}). The penalty term \eqref{eq:overlap_penalty} effectively discourages the case where the 2-norm ball $\calB(\bfp_i, \delta_q)$ centered around $\bfp_i$ with radius $\delta_q$ coincides with $\calB(\bfp_j, \delta_q)$. Since $\calB(\bfp_i, \delta_q)$ contains $\bbU(\bfX_i)$, $q(\bfp_i, \bfp_j) = 0$ is sufficient to ensure Cond.~\eqref{cond:2} is not violated for a pair of poses $\bfX_i$ and $\bfX_j$. Note that Cond.~\eqref{cond:1} is an inherent property of the weighting function $\alpha_\bfV(\bfX)$ and can be evaluated offline. Putting all the components together, the differentiable objective function for gradient-based active mapping is expressed as follows:
%
\begin{equation}
\begin{aligned}
    &f(\bfX_{t+1:t+T}) =\\
    &\sum_{\tau=1}^T \Bigg( \sum_{\bfV \in \bar{\calG}(\bfX_{t+\tau})} \alpha_\bfV(\bfX_{t+\tau}) [I(\bfm; \bfz \mid \bfV, \calH_t) +\\
    &\gamma_c \log(d(\bfv, p_t(\bfm)))] - \frac{\gamma_q}{2} \sum_{\substack{\tau'=1 \\ \tau' \neq \tau}}^T q(\bfp_{t+\tau}, \bfp_{t+\tau'}) \Bigg),
\end{aligned}
\label{eq:obj_func}
\end{equation}
where $\bfv$, $\bfp_{{t+\tau}}$, and $\bfp_{{t+\tau'}}$ are the corresponding positions of poses $\bfV$, $\bfX_{t+\tau}$, and $\bfX_{t+\tau'}$.

\begin{figure}[t]
  \hbox{\hspace{4ex}\includegraphics[width=0.9\linewidth]{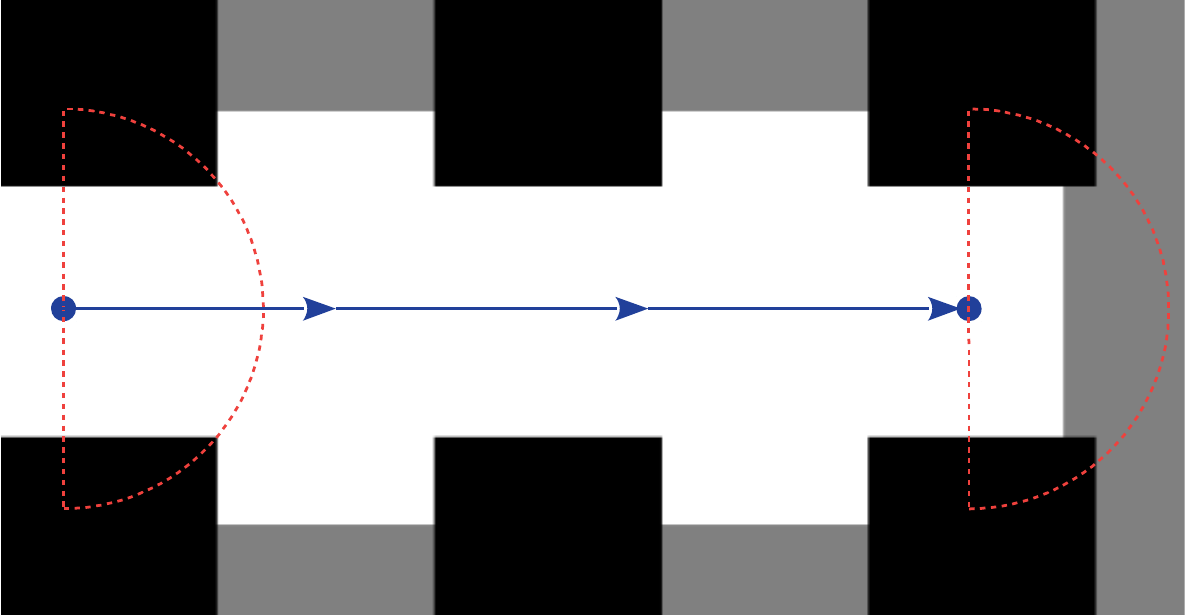}}
  \centering
  \includegraphics[width=0.96\linewidth]{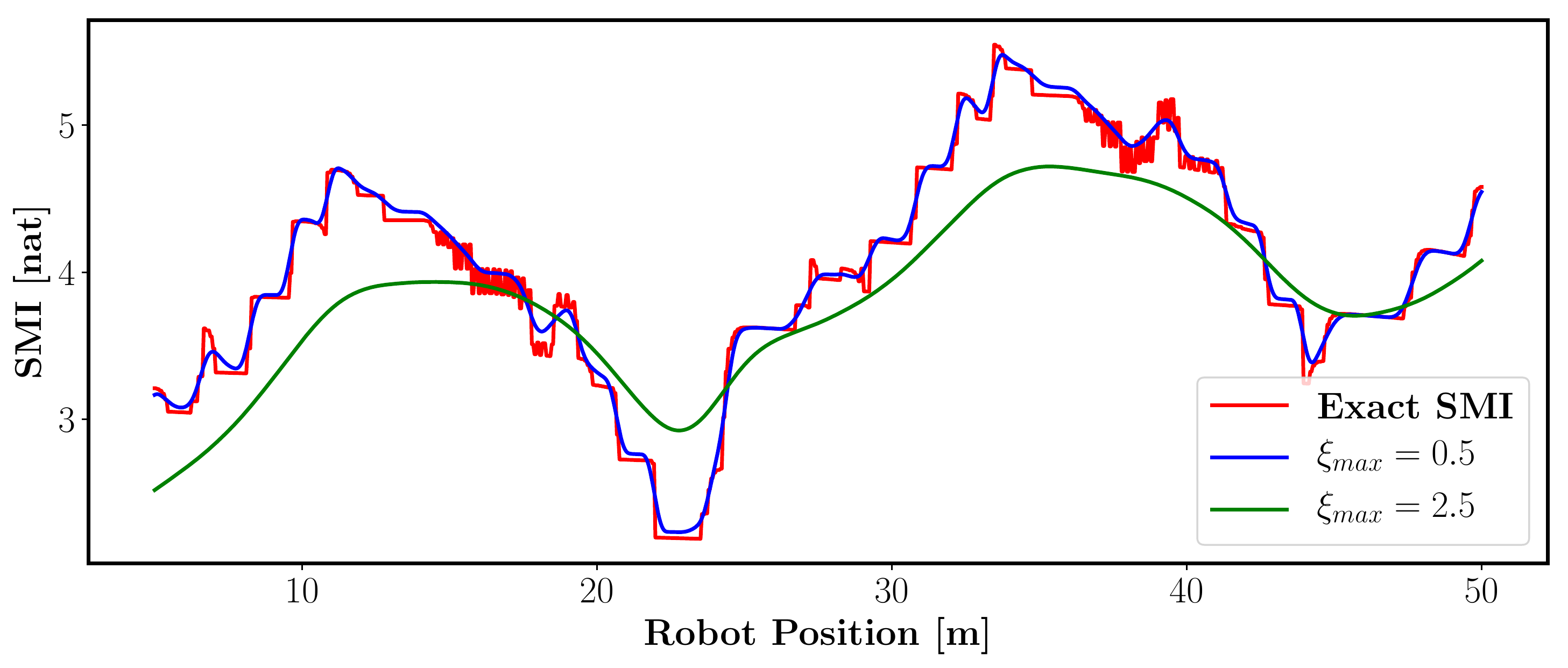}
  \caption{Example of the SMI approximation using \eqref{eq:alpha}. Top: Partially known environment where black, white, and gray regions depict occupied, free, and unknown areas. The robot equipped with a $180^{\circ}$ range sensor moves along a straight line shown as the blue arrowed segment. Bottom: Exact value of the SMI evaluated at each robot position alongside approximate SMI values for various radii $\xi_{max}$.}
  \label{fig:approx_smi_path}
\end{figure}

So far we assumed a general definition for the differentiable weighting function $\alpha_{\bfV}(\bfX)$ that satisfies Cond.~\ref{cond:1} and \ref{cond:2}. We use the following definition for $\alpha_\bfV(\bfX)$:
\begin{equation}
\begin{aligned}
    &\alpha_\bfV(\bfX) =\frac{\upsilon(\delta(\bfxi_{\bfX, \bfV})) (1 + \cos{\delta(\bfxi_{\bfX, \bfV})})}{\sum_{\bfU \in \calG} \upsilon(\delta(\bfxi_{\bfX, \bfU})) (1 + \cos{\delta(\bfxi_{\bfX, \bfU})})},
\end{aligned}
\label{eq:alpha}
\end{equation}
where $\delta(\bfxi_{\bfX, \bfV})$ is the distance between poses $\bfV,\bfX \in \!SE(3)$:
\begin{equation}
    \delta(\bfxi_{\bfX, \bfV}) = \frac{\pi}{\xi_{{\rm max}}} \sqrt{\bfxi_{\bfX, \bfV}^\top \Gamma \bfxi_{\bfX, \bfV}}, \quad \bfxi_{\bfX, \bfV} = \log(\bfX^{-1} \bfV)^\vee.
\label{eq:SE3-dist}
\end{equation}
Here, $\Gamma$ is a diagonal matrix containing positive coefficients and $\bfxi_{\bfX, \bfV}$ is the difference of two \textit{SE(3)} poses in the local frame of $\bfX$. The indicator function $\upsilon(\delta(\bfxi_{\bfX, \bfV}))$ in \eqref{eq:alpha} is equal to one only when $0 \leq \delta(\bfxi_{\bfX, \bfV}) \leq \pi$ and zero otherwise, which in effect limits the SMI approximation to the viewpoints within the radius $\xi_{{\rm max}}$ from $\bfX$, satisfying Cond.~\ref{cond:1}. Note that the discontinuity of $\upsilon(\delta(\bfxi_{\bfX, \bfV})$ occurs at the same point where $1 + \cos{\delta(\bfxi_{\bfX, \bfV})} = 0$; hence $\alpha_\bfv(\bfX)$ is differentiable with respect to all poses $\bfX \in \textit{SE(3)}$. Fig.~\ref{fig:approx_smi_path} shows an example of the accurate SMI evaluation at different robot positions compared to the approximate SMI of \eqref{eq:SMI_approx_single} using the weighting function of \eqref{eq:alpha}. Note that, while being differentiable, the approximations follow the occlusion-aware behavior of the exact SMI, i.e. they peak at the positions where more of the unexplored region is visible. With $\alpha_\bfv(\bfX)$ specified as \eqref{eq:alpha}, we can compute a closed-form expression for the gradient of the approximate SMI as follows.
\begin{proposition}
\label{prop:approx_smi_grad}
Using the weighting function of \eqref{eq:alpha}, the gradient of the approximate SMI with respect to the robot pose can be obtained as follows:
\begin{equation}
\begin{aligned}
     &\nabla_{\bfpsi} \Tilde{I}(\bfm; \bfz \mid \bfX \exp{(\hat{\bfpsi})}, \calH_t)|_{\bfpsi = 0} 
     = \left(\frac{\pi}{\xi_{{\rm max}} \eta(\bfX)} \right)^2\\
     &\qquad\qquad\qquad\qquad \times \calJ_R^\top(\bfxi_\bfX) \sum_{\bfV \in \Bar{\calG}(\bfX)} \Lambda_\bfV(\bfX) \bfxi_{\bfX, \bfV},\\
     &\Lambda_\bfV(\bfX)  = [\eta(\bfX) I(\bfm; \bfz \mid \bfV, \calH_t) - \beta(\bfX)] \\
     &\qquad\qquad\qquad\qquad \times \frac{\sin{\delta(\bfxi_{\bfX, \bfV})}}{\delta(\bfxi_{\bfX, \bfV})} \calJ_R^\top(-\bfxi_{\bfX, \bfV}) \Gamma,\\
     &\eta(\bfX) = |\Bar{\calG}(\bfX)| + \sum_{\bfU \in \Bar{\calG}(\bfX)} \cos{\delta(\bfxi_{\bfX, \bfU})},\\
     &\beta(\bfX) = \sum_{\bfU \in \Bar{\calG}(\bfX)} I(\bfm; \bfz \mid \bfU, \calH_t) (1 + \cos{\delta(\bfxi_{\bfX, \bfU})}),\\
     &\bfxi_\bfX = \log(\bfX)^\vee,
\end{aligned}
\label{eq:approx_smi_grad}
\end{equation}
\noindent where $\Bar{\calG}(\bfX)$ is the subset of viewpoints $\bfV \in \calG$ that $\upsilon(\delta(\bfxi_{\bfX, \bfV})) = 1$, and $\calJ_R(.)$ is the right Jacobian of \textit{SE(3)}.
\end{proposition}
\begin{proof}
See Appendix~\ref{app:approx_smi_grad}.
\end{proof}
%
The objective function \eqref{eq:obj_func} can used with any Bayesian sensing model \eqref{eq:bayes_rule}. In the context of occupancy grid mapping via a range sensor, the log-odds technique using a narrow beam \cite[Ch.~9]{ProbabilisticRoboticsBook} is a natural choice for the sensing model which provides a computationally simple yet accurate representation of the environment. Using such model, the SMI between the grid map and a beam $z$ can be written as a weighted sum of occlusion probabilities~\cite{julian, ssmi_icra, fsmi}:
\begin{equation}
\begin{aligned}
    I(\bfm; z \mid \bfX, \calH_t) = \sum_{i =1}^{N_{\bfX}} w_i(\calH_t) p(i \mid \calH_t)
\end{aligned}
\label{eq:smi_conceptual}
\end{equation}
where $N_{\bfX}$ is the maximum number of map cells that can be visited by a sensor beam emitted from pose $\bfX$ and $p(i | \calH_t)$ is the probability that the beam is occluded at $i$-th cell along its path. Hence, the corresponding SMI expression is occlusion-aware, namely the value of the mutual information is directly controlled by whether or not the observations are occluded.

The benefit of using the occlusion-aware SMI formulation of \eqref{eq:smi_conceptual} comes with the computational burden of ray-tracing during each evaluation of the SMI, since \eqref{eq:smi_conceptual} requires identifying the map cells along the observation beam. Going back to the objective function \eqref{eq:obj_func}, it is required to perform $B \sum_{\tau=1}^T |\hat{\calG}(\bfX_{t+\tau})|$ ray-tracings for a range sensor with $B$ beams during every evaluation of $f(\bfX_{t+1:t+T})$. Since all viewpoints in $\calG$ are located at map cell centers, the ray-tracing for a viewpoint $\bfV$ with position $\bfv$ and orientation $\bfR$ can be obtained from a ray-tracing from the map origin with the same orientation and simply translated by $\bfv$. Considering that all viewpoints in $\calG$ have a fixed set of orientations $\calR$, we skip online ray-tracing by computing $|\calR|$ ray-tracings from the origin only once and query ray-tracings by applying translation $\bfv$ for each viewpoint $\bfV$. To further accelerate the optimization, we avoid repeated evaluations of $I(\bfm; \bfz \mid \bfV, \calH_t)$ by storing the values for each viewpoint $\bfV$ until the map is updated. For computing the free distance $d(\bfv, p_t(\bfm))$, we once again use the fact that all viewpoints are located at map cell centers; therefore we can obtain $d(\bfv, p_t(\bfm))$ from the distance transform of $\hat{\calE}_f(p_t(\bfm))$ in Def.~\ref{def:free_dist} scaled by the map resolution. This needs to be computed only once for each planning step, since the distance transform provides values of $d(\bfv, p_t(\bfm))$ for all cell centers $\bfv$.

Gradient-based optimization of the objective function \eqref{eq:obj_func} allows local maximization of trajectory informativeness as well as distance from obstacles. Given an initial trajectory which can be provided by frontier-based exploration~\cite{frontier}, we perform gradient ascent for $n_{\rm max}$ steps or until the improvement in $f(\bfX_{t+1:t+T})$ is less than $0.1\%$. The output of this optimization will be a kinematically feasible trajectory $\bfX_{t+1:t+T}^* \in \textit{SE(3)}^T$ which can be tracked by a low-level controller specified by the robot dynamics. Alg.~\ref{alg:active_mapping} summarizes our procedure for occlusion and collision-aware active mapping. We evaluate the performance of the proposed method in the next section.

\begin{algorithm}[t]
\caption{Active Mapping via Gradient Ascent}\label{alg:active_mapping}
\begin{algorithmic}[1]
  \renewcommand{\algorithmicrequire}{\textbf{Input:}}
  \renewcommand{\algorithmicensure}{\textbf{Output:}}
  \Require $\bfX_t$, $p_t(\bfm)$, $T$, $n_{\rm max}$, $pcrt$
  \If{$pcrt$ is None}
    \State $pcrt = \Call{preComputeRayTracing}{ }$
  \EndIf
  \State $\bfd = \Call{distanceTransform}{p_t(\bfm)}$
  \State $k \leftarrow 0$
  \State $\bfX^{(k)}_{t+1:t+T} = \Call{initPath}{\bfX_t, p_t(\bfm), T}$
  \While{not $\Call{terminate}{f(\bfX^{(k)}_{t+1:t+T}), n_{\rm max}}$}
  \State $\bfX^{(k+1)}_{t+1:t+T} = \Call{gradAscent}{f(\bfX^{(k)}_{t+1:t+T}), pcrt, \bfd}$
  \State $k \leftarrow k + 1$
  \EndWhile
  \State \Return $\bfX^{(k)}_{t+1:t+T}$
\end{algorithmic}
\end{algorithm}
\section{Experiments}
\label{sec:experiments}

In this section we evaluate the performance of our proposed active mapping method in several simulated and real-world experiments in comparison to baseline exploration strategies: frontier-based exploration (Frontier)~\cite{frontier}, FSMI~\cite{fsmi}, SSMI~\cite{ssmi_icra}, and optimized next best view using RRT (O-NBV-\textit{RRT})~\cite{DiAIM}. In Sec.~\ref{subsec:2D-exp-sim}, we perform 2-D active mapping using the proposed method alongside the baselines in a set of $10$ randomly generated 2-D environments. Sec.~\ref{subsec:3D-exp-sim} contains large-scale active mapping using OctoMap representation~\cite{octomap} in a simulated 3-D Unity environment. Lastly, in Sec.~\ref{subsec:real-world-exp} we demonstrate the performance of our method in a real-world environment using a wheeled ground robot.

Across all experiments, each method uses the log-odds mapping from range measurements \cite[Ch.~9]{ProbabilisticRoboticsBook} but select robot trajectories based on their own criteria. Moreover, during each planning step for 2-D exploration, we identify the frontiers using edge detection over the most likely map at time $t$, i.e. the mode of $p_t(\bfm)$. Then, we cluster the frontier map cells by detecting the connected components. We plan a path from robot pose $\bfX_t$ to the center of each frontier using $A^*$ graph search to produce candidate paths for FSMI and SSMI. For Frontier, we pick the path that maximizes the ratio between the frontier size and path length. For O-NBV-\textit{RRT}, we pick the goal of the same path used by Frontier, but instead plan an \textit{RRT} path towards the goal. Our method also uses the same path used by Frontier as the initial trajectory. For exploration in 3-D environments, we first project the most likely OctoMap at time $t$ onto the ground level in order to derive a 2-D occupancy map and proceed with similar steps as in 2-D exploration. We selected maximum size of neighboring viewpoints $\xi_{\rm max} = 2$, maximum number of gradient ascent iterations $n_{\rm max} = 50$, step size $l = 10$, collision penalty weight of $\gamma_c = 5 \times 10^{-4}$, and $\gamma_q = 1$. For $\Gamma$ in \eqref{eq:SE3-dist}, we use a diagonal matrix with $[1, 1, 0.1]$ as the elements on the diagonal.

\subsection{2-D Active Mapping in Simulation}
\label{subsec:2D-exp-sim}

\begin{figure}[t]
    \includegraphics[width=\linewidth]{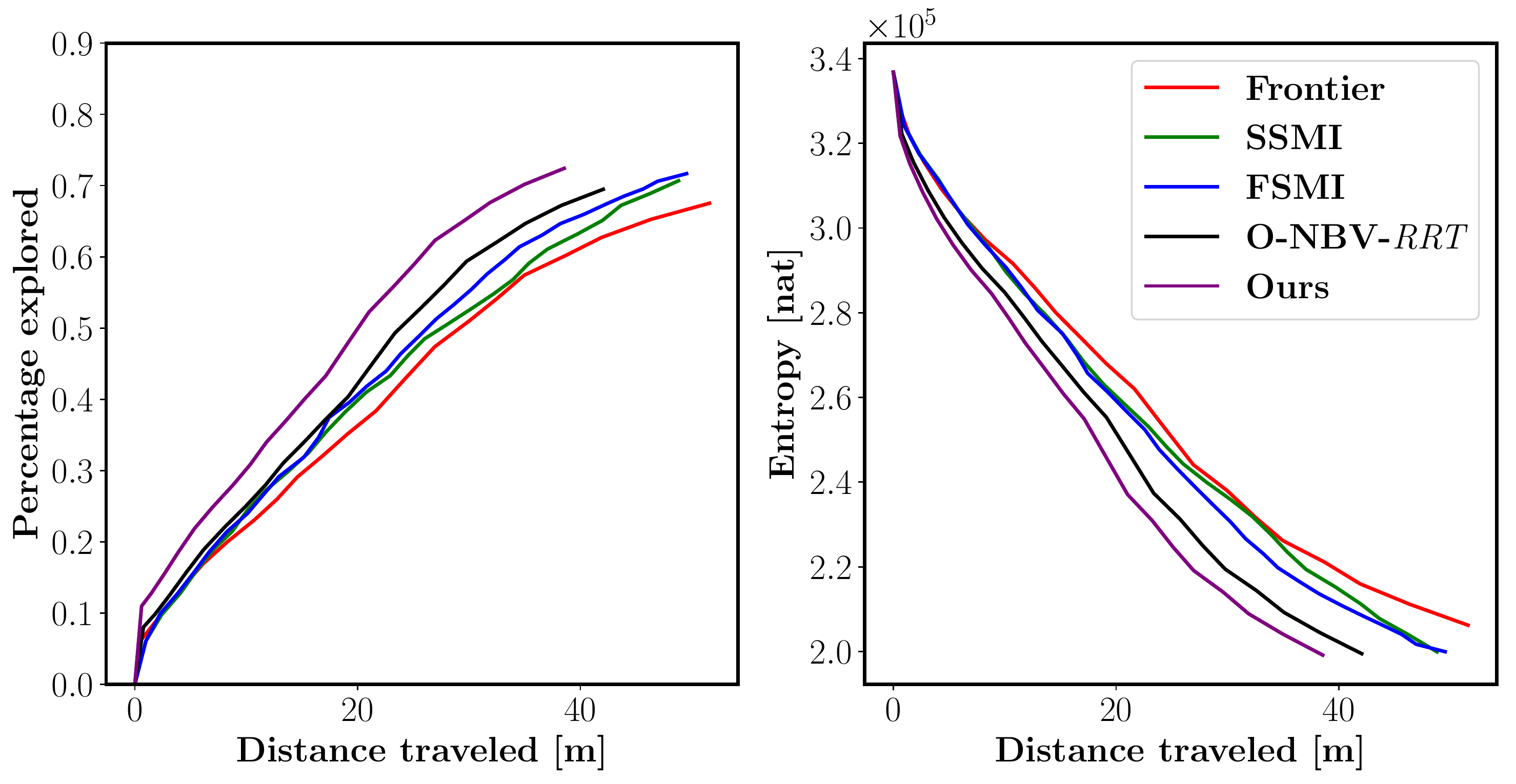}
    \caption{2-D simulated active mapping performance compared among various active mapping strategies. The results are averaged over $10$ randomly generated environments with $3$ random starting positions for each environment.}
    \label{fig:2D-exp-sim-comparison}
\end{figure}

Here we compare our proposed method to Frontier, FSMI, SSMI, and O-NBV-\textit{RRT}. The experiments are performed in $10$ randomly generated 2-D environments with dimensions $60m \times 60m$ and $3$ random starting positions for each instance. We consider a robot equipped with a LiDAR sensor of range $10m$ and $90^{\circ}$ field of view where each measurement beam is added with Gaussian noise of $\calN(0, 0.1)$. Fig.~\ref{fig:2D-exp-sim-comparison} shows the 2-D simulation results among all methods. We witness superior exploration performance of our proposed method, which can be attributed to the fact that our method applies occlusion-aware optimization for each pose in a candidate trajectory; resulting in high utilization of information gathering opportunities. The subpar performance of  O-NBV-\textit{RRT} can be caused by the fact that O-NBV-\textit{RRT} does not account for sensor noise during pose optimization. Robot trajectories obtained by Frontier, FSMI, and SSMI are computed using $A^*$ search, and no further optimization is performed over single robot poses along the trajectory; resulting in a very restricted set of candidate solutions.

\subsection{Exploration in 3-D Unity Simulation}
\label{subsec:3D-exp-sim}

\begin{figure*}[t]
    \begin{subfigure}[t]{0.30\linewidth}
    \includegraphics[width=\linewidth]{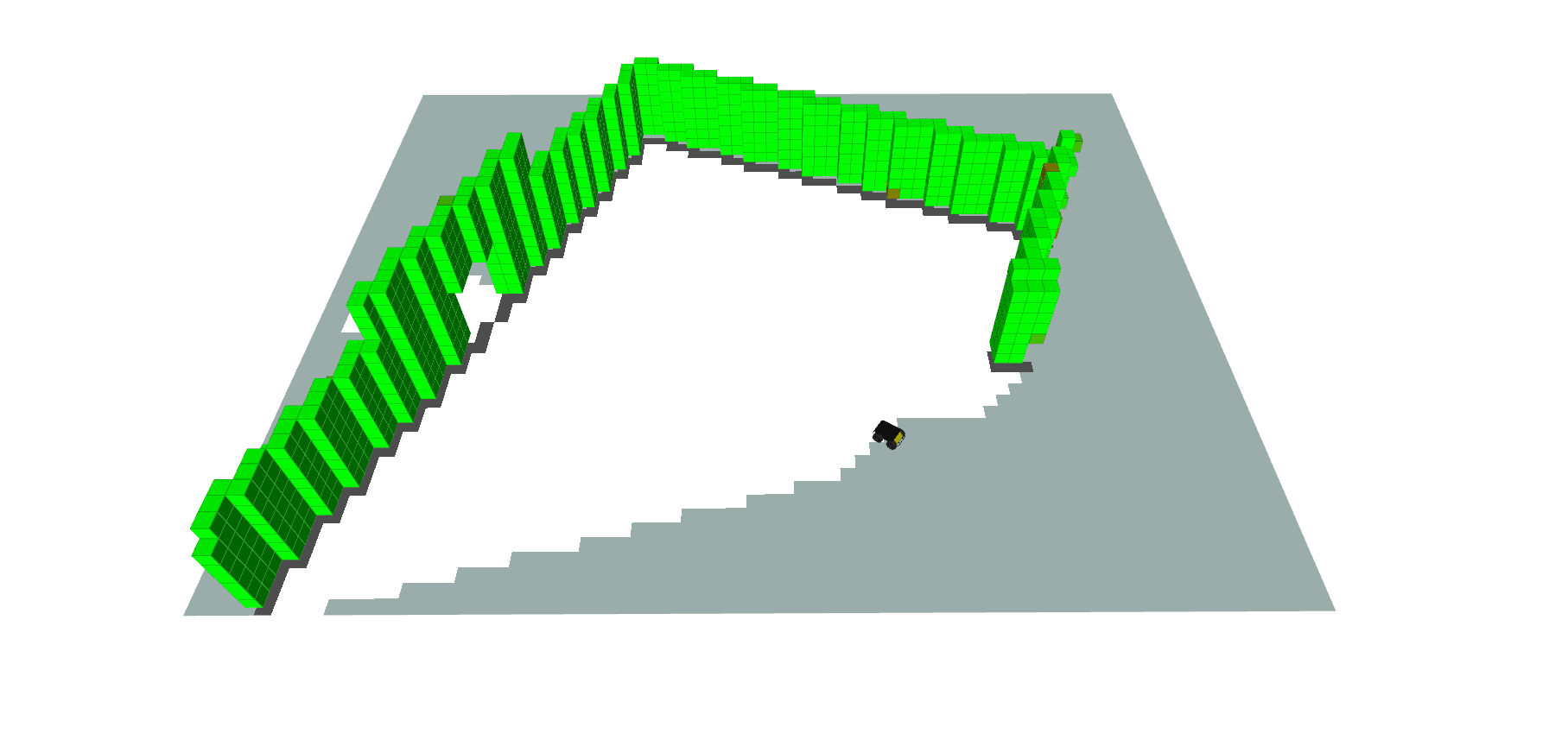}
    \captionsetup{justification=centering}
    \caption{The robot begins exploration.}
    \label{fig:3D-exp-sim-timelapse-a}
    \end{subfigure}%
    \hfill%
    \begin{subfigure}[t]{0.30\linewidth}
    \includegraphics[width=\linewidth]{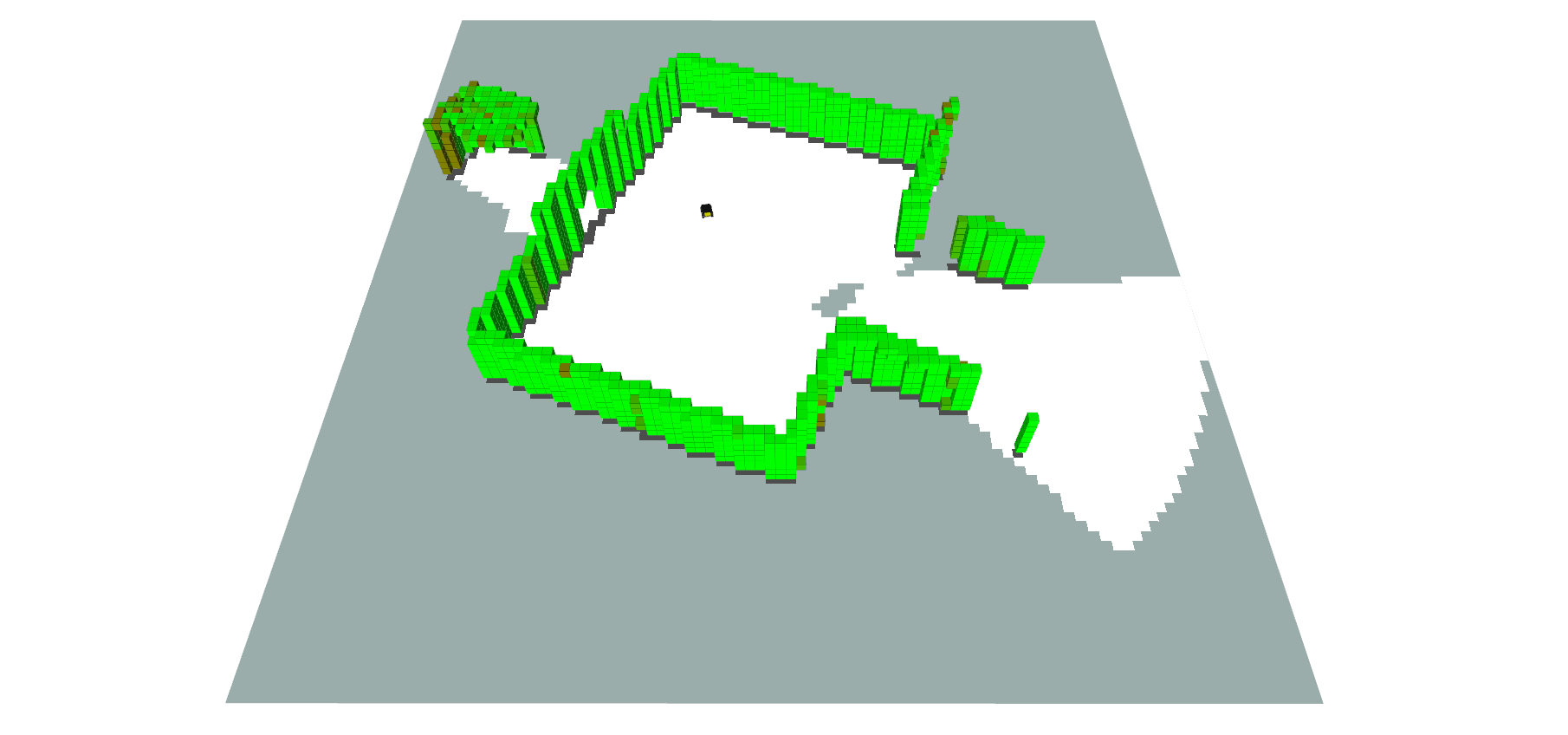}
    \caption{During the first $3$ iterations, the robot tries to build the map of its immediate vicinity.}
    \label{fig:3D-exp-sim-timelapse-b}
    \end{subfigure}%
    \hfill%
    \begin{subfigure}[t]{0.30\linewidth}
    \includegraphics[width=\linewidth]{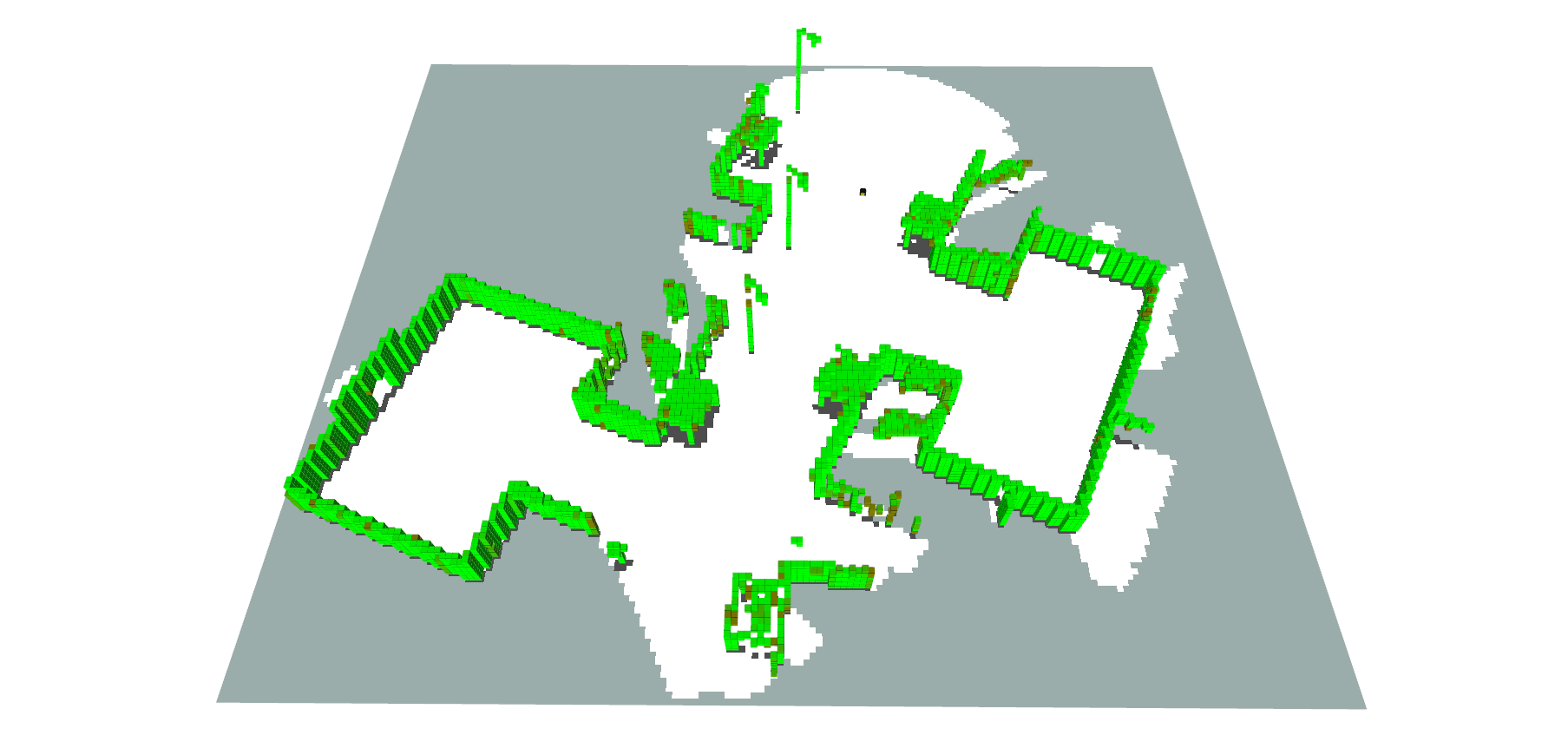}
    \caption{The robot explores the boundaries of the known region at iteration $10$.}
    \label{fig:3D-exp-sim-timelapse-c}
    \end{subfigure}\\
    \begin{subfigure}[t]{0.45\linewidth}
    \centering
    \includegraphics[width=\linewidth]{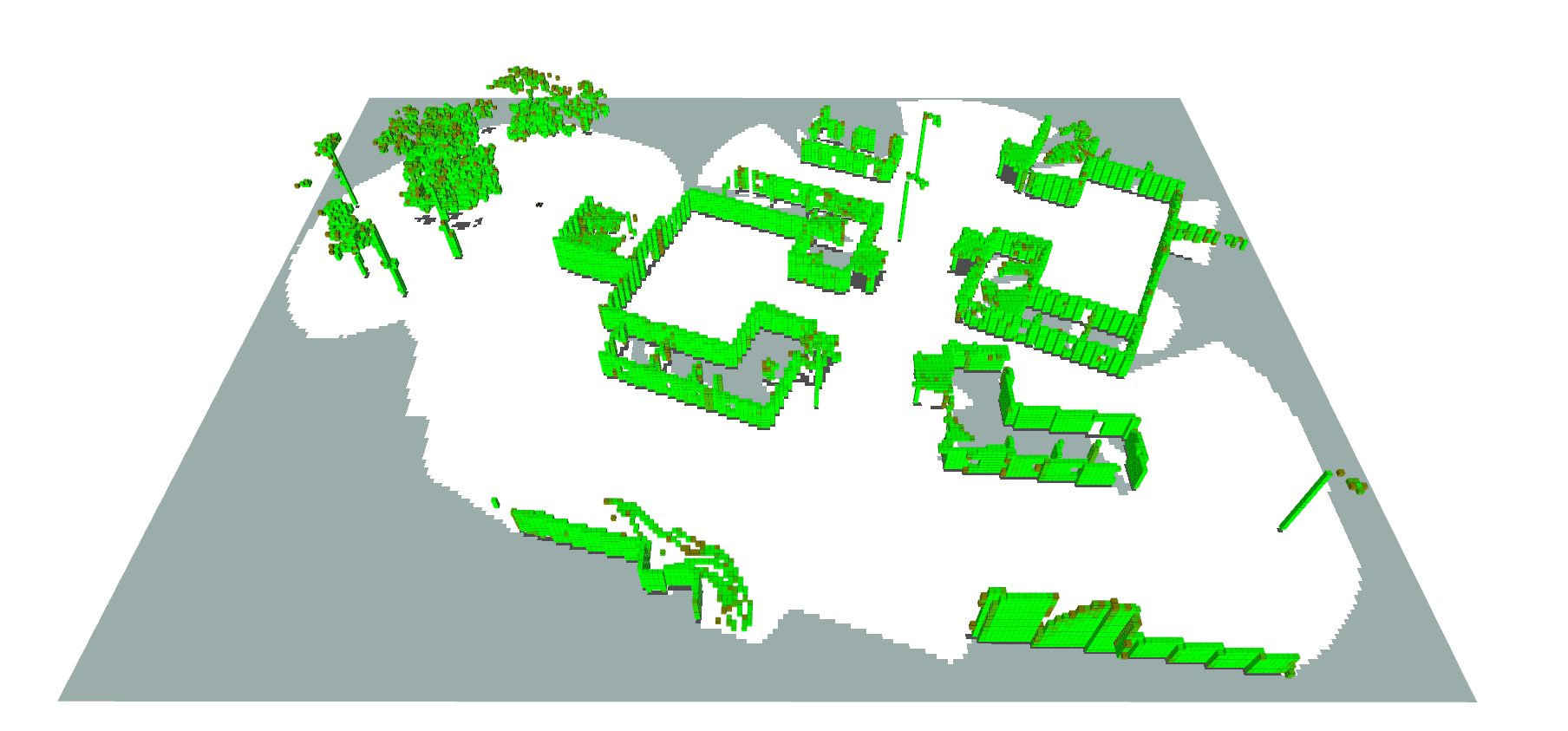}
    \captionsetup{justification=centering}
    \caption{OctoMap after $20$ exploration iterations}
    \label{fig:3D-exp-sim-timelapse-d}
    \end{subfigure}%
    \hfill%
    \begin{subfigure}[t]{0.45\linewidth}
    \centering
    \includegraphics[width=\linewidth]{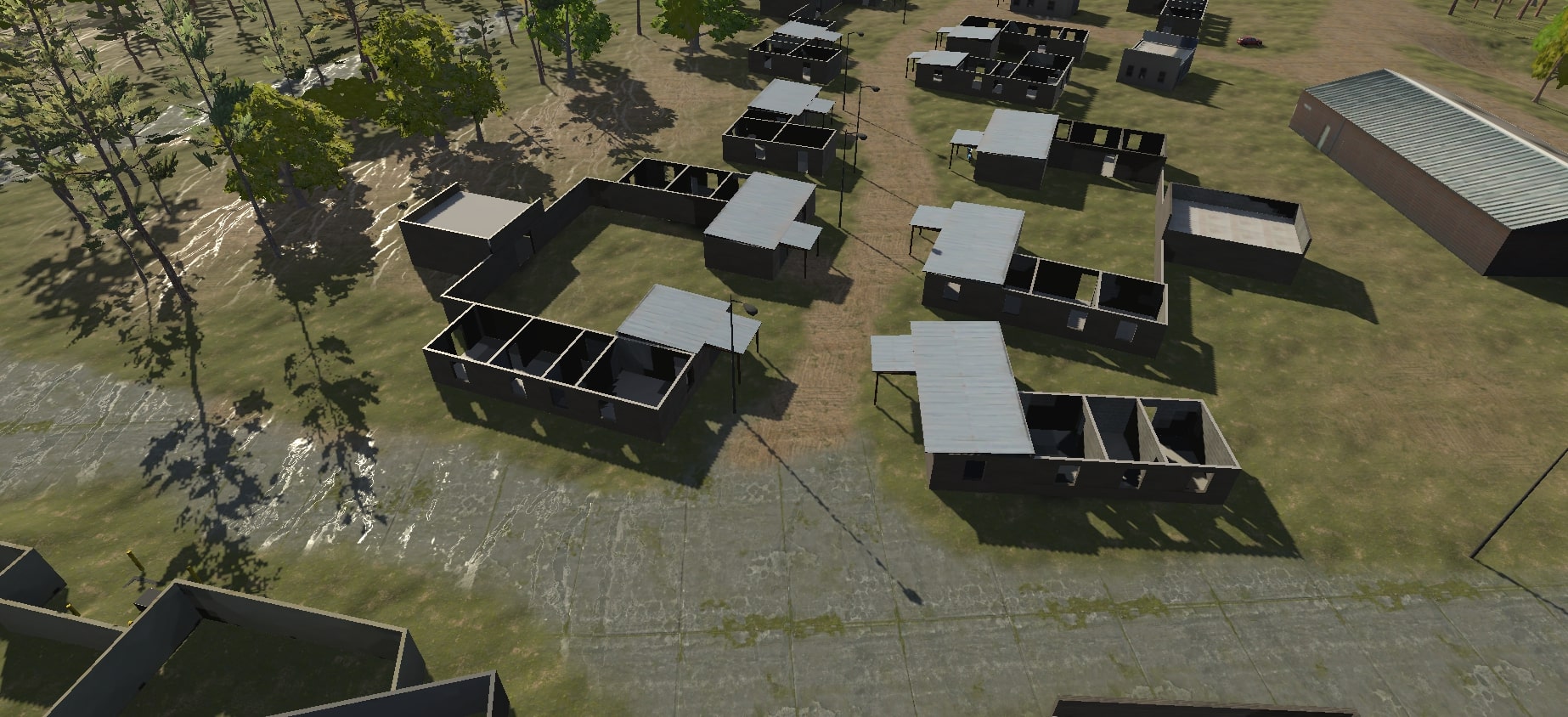}
    \captionsetup{justification=centering}
    \caption{Photo-realistic Unity simulation environment}
    \label{fig:3D-exp-sim-timelapse-e}
    \end{subfigure}
    \caption{Time lapse of 3-D active mapping using the proposed method in simulated Unity environment. The robot receives depth measurements using an RGBD camera, and incrementally builds an OctoMap as it explores the unknown environment.}
    \label{fig:3D-exp-sim-timelapse}
\end{figure*}

In this part, we test our proposed method in a photo-realistic 3-D Unity simulated environment, shown in Fig.~\ref{fig:3D-exp-sim-timelapse-e}. We use a Husky robot equipped with a depth camera that provides a 3-D point cloud used for building an OctoMap. OctoMap~\cite{octomap} is occupancy grid mapping method based on OcTree data structure and the log-odds technique that provides a scalable way to store the 3-D map for large environments through compressing the map cells with similar occupancy probability. For experiments in 3-D environments, we operate our proposed method in two different modes. In the first mode (Exp-2D), we project both the OctoMap and 3-D point clouds onto the $z = 0$ surface, obtaining a 2-D representation of the map and the observations. Then we compute informative trajectories similar to the 2-D active mapping in Sec.~\ref{subsec:2D-exp-sim}. By executing the optimal trajectory, we receive new 3-D point cloud observations that incrementally update the OctoMap. The first mode of active mapping only requires ray-tracing in a 2-D grid map, saving computation time while trading for accuracy in evaluating mutual information. In the second mode (Exp-3D), we directly use the 3-D sensor model for ray-tracing within the OctoMap. This mode of operation is expected to evaluate mutual information more accurately since the exact sensor model and 3-D occlusion has been taken into account. Fig.~\ref{fig:3D-exp-sim-timelapse} shows several exploration iterations of active mapping process using Exp-3D. Fig.~\ref{fig:3D-exp-comparison}.~(top) shows simulation results for experiments in the 3-D Unity environment for $20$ exploration iterations. We observe Exp-2 finds a good balance between exploration efficiency and computation time; while Exp-3D is slightly more efficient in terms of distance traveled, it takes longer computation time to perform exploration compared to Exp-2D. This can be attributed to the fact that the objects in the simulation environment have a uniform profile in the $z$ direction, hence the amount of information in the 2-D and 3-D maps are almost the same. TABLE.~\ref{tbl:3D-exp-clearance}~(middle column) compares the average clearance from the obstacles among different methods. Compared to Frontier, it is clear that the proposed method chooses informative trajectories that avoid approaching the obstacles.

\begin{figure}[t]
    \includegraphics[width=\linewidth]{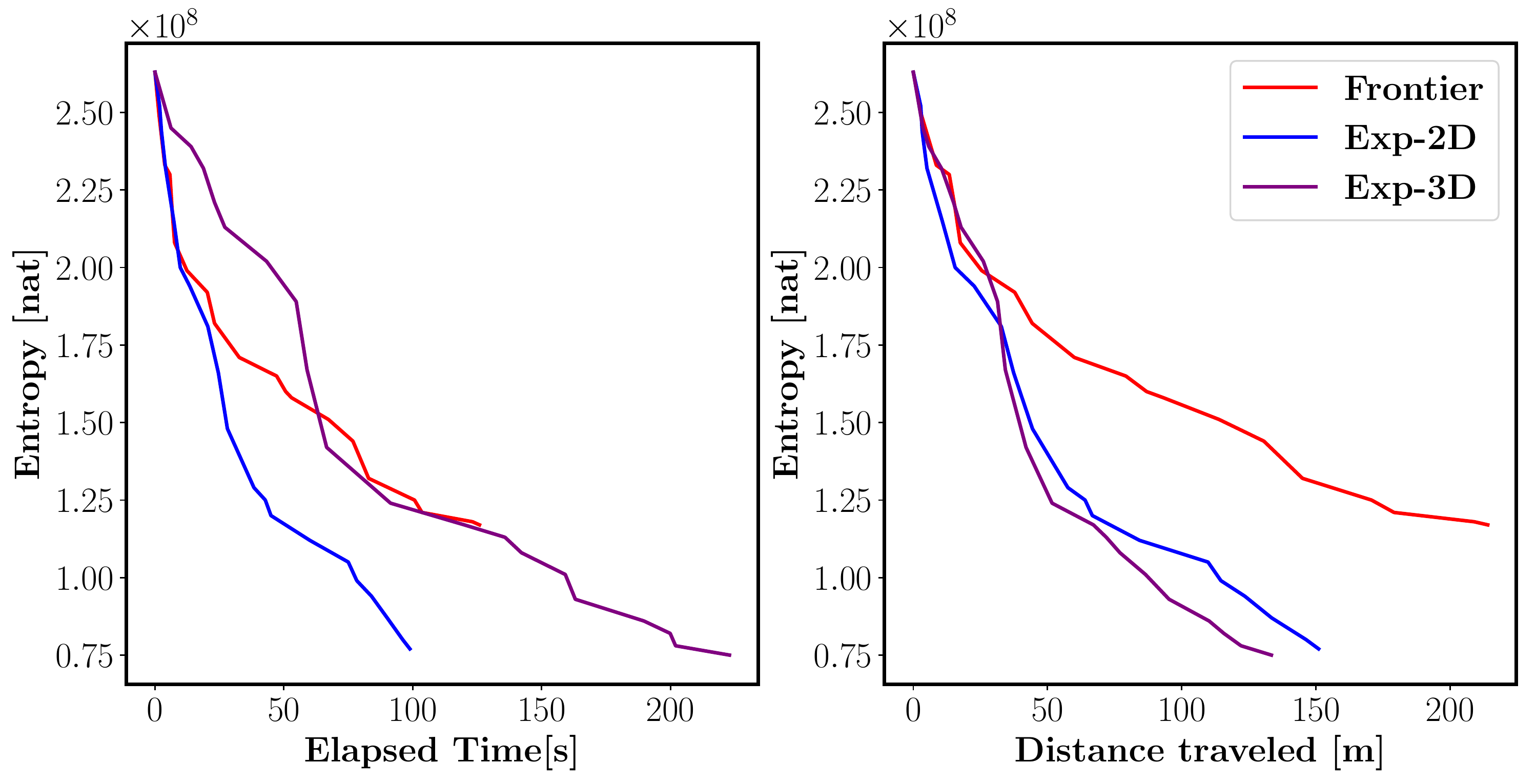}\\
    \includegraphics[width=\linewidth]{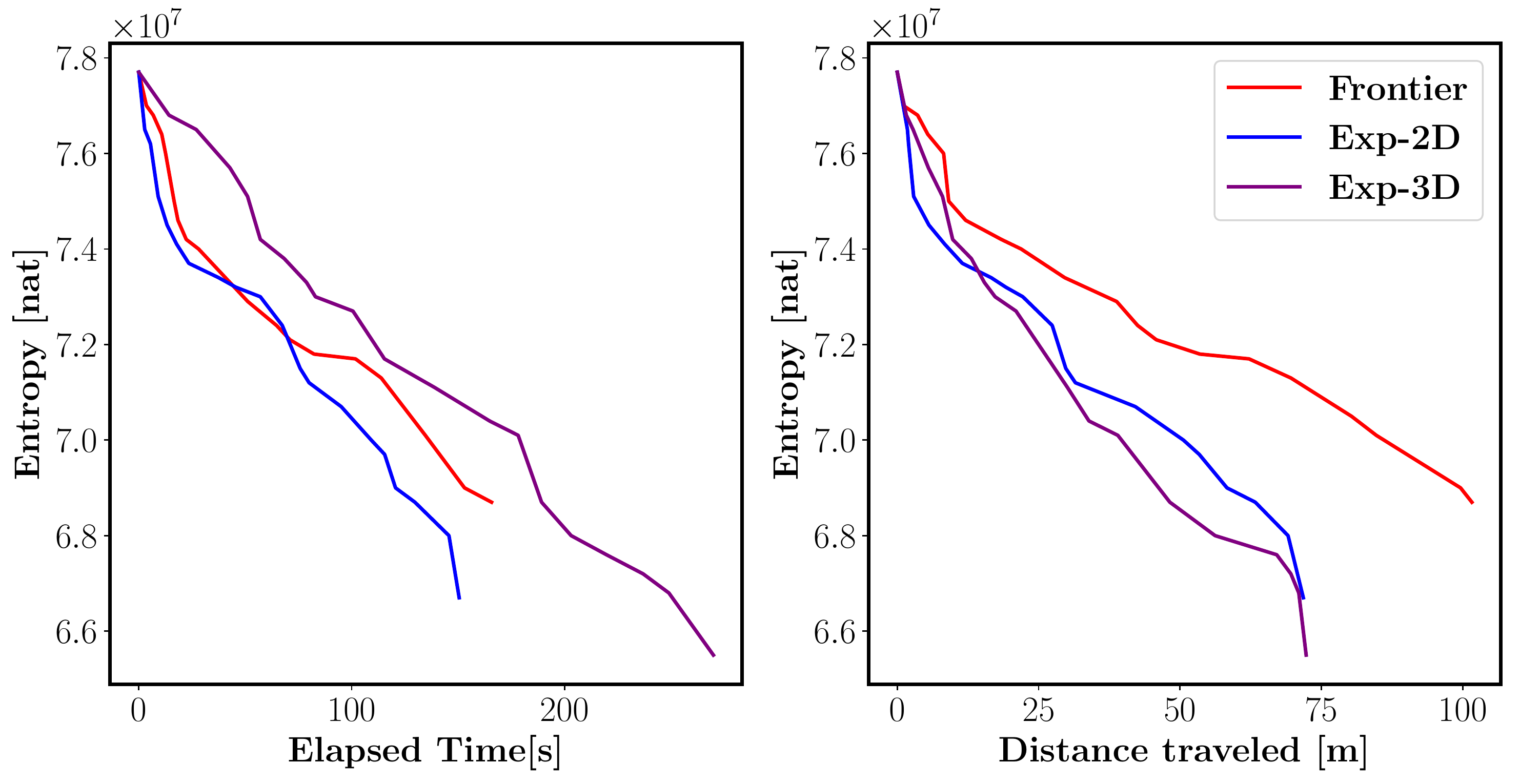}
    \caption{3-D active mapping performance compared between Frontier and our method for $20$ exploration iterations. Exp-2D uses a projected 2-D map and observation model for planning, while Exp-3D performs ray-tracing and information computation directly on the OctoMap. Top row: Unity simulation environment. Bottom  row: Real-world experiments.}
    \label{fig:3D-exp-comparison}
\end{figure}

\begin{table}[t]
\caption{Clearance from obstacles compared between Frontier and our proposed method averaged over $20$ exploration iterations.}
\label{tbl:3D-exp-clearance}
\centering
\begin{tabular}{|p{0.5in}|p{1.0in}|p{1.0in}|}
\hline
Algorithm & 3-D Unity Simulation & Real World\\
\hline
Frontier & 1.2 & 0.48\\
Exp-2D & \textit{\textbf{3.9}} & 1.7\\
Exp-3D & 3.4 & \textit{\textbf{1.84}}\\
\hline
\end{tabular}
\end{table}

\subsection{Real-world Experiments}
\label{subsec:real-world-exp}

We deployed Exp-2D and Exp-3D on a ground wheeled robot to autonomously explore an indoor environment. The robot was equipped with a NVIDIA Xavier NX GPU, a Hokuyo UST-10LX LiDAR, and an Intel RealSense D435i RGBD camera. We implemented our software stack using \textit{Robot Operating System} (ROS)~\cite{ros}. Robot localization was carried out using ICP scan matching of LiDAR measurements~\cite{icp}. 3-D point clouds from the depth images were used to build an OctoMap. The complete implementation was able to update the OctoMap every $0.11 s$ on average. Fig.~\ref{fig:3D-exp-comparison}.~(bottom) and TABLE.~\ref{tbl:3D-exp-clearance}~(right column) show the performance of active mapping for $20$ exploration iterations and the average clearance from obstacles, respectively. The real-world experiments confirm the findings in the simulations in terms of the efficiency of the proposed method.
\section{Conclusion}
\label{sec:conclusion}

This paper developed a differentiable approximation of the Shannon mutual information between a probabilistic occupancy grid map and range sensor measurements. Our formulation enables gradient-based optimization of informative occlusion-aware sensing trajectories in 3-D and allows the inclusion of additional differentiable penalty terms, such as collision cost. We demonstrated in simulated and real-world experiments that our method outperforms the state-of-the-art techniques due to its ability to optimize the sensing views in continuous space. In a future work, we will investigate extending the proposed method to multiple agents, where we expect gradient-based optimization to also be significantly more efficient than discrete space search.


\appendices
\section{Proof of Prop.~\ref{prop:approx_smi_multi}}
\label{app:approx_smi_multi}

Cond.~\eqref{cond:1} and \eqref{cond:2} effectively state that for any pair of viewpoints $\bfV \in \hat{\calG}(\bfX_i)$ and $\bfU \in \hat{\calG}(\bfX_j)$, the two sets of map cells inside the FOVs of $\bfV$ and $\bfU$ do not intersect. This is true since the spaces inside $\bfV \ubar{\calF}$ and $\bfU \ubar{\calF}$ are always subsets of $\bbU(\bfX_i)$ and $\bbU(\bfX_j)$, respectively, while Cond.~\eqref{cond:2} states that $\bbU(\bfX_i) \cap \bbU(\bfX_j) = \emptyset$. Consequently, the observations made from $\bfV$ and $\bfU$ are independent random variables, resulting in the following decomposition of the SMI:
\begin{equation}
\begin{aligned}
    I(\bfm; \bfz_v, \bfz_u &\mid \bfV, \bfU, \calH_t) =\\
    &I(\bfm; \bfz_v \mid \bfV, \calH_t) + I(\bfm; \bfz_u \mid \bfU, \calH_t). \notag
\end{aligned}
\end{equation}
Following Cond.~\eqref{cond:2}, the above decomposition can be applied for any set of viewpoints $\calV := \bfV_{t+1:t+T}$ where $\bfV_{t+\tau} \in \hat{\calG}(\bfX_{t+\tau})$, $\tau \in \crl{1, \ldots, T}$. Hence we have:
\begin{equation}
\begin{aligned}
    &\sum_{\tau = 1}^T \sum_{\calV \in \calG^T} A_{\calV}(\bfX_{t+1:t+T}) I(\bfm; \bfz_{t+\tau} \mid \bfV_{t+\tau}, \calH_t) =\\
    &\sum_{\tau = 1}^T \Tilde{I}(\bfm; \bfz_{t+\tau} \mid \bfX_{t+\tau}, \calH_t) \sum_{\calV - \bfV_{t+\tau} \in \calG^{T-1}} \frac{A_{\calV}(\bfX_{t+1:t+T})}{\alpha_{\bfV_{t+\tau}}(\bfX_{t+\tau})}. \notag
\end{aligned}
\end{equation}
Based on the definition of $A_{\calV}(\bfX_{t+1:t+T})$, the inner sum is equal to $1$, which yields the expression in \eqref{eq:SMI_approx_multi}.\qed


\section{Proof of Prop.~\ref{prop:approx_smi_grad}}
\label{app:approx_smi_grad}

Since the approximate SMI in \eqref{eq:SMI_approx_single} is linear with respect to $\alpha_\bfV(\bfX)$ terms, the overall gradient computation can be reduced to a weighted sum of individual gradients of $\alpha_\bfV(\bfX)$ with respect to robot pose $\bfX$. Also, it is only needed to compute gradients for viewpoints where $\upsilon(\delta(\bfxi_{\bfX, \bfV})) = 1$ since the rest of the viewpoints do not affect the derivations:
\begin{equation}
\begin{aligned}
    &[\nabla_{\bfpsi} \alpha_\bfV(\bfX \exp{(\hat{\bfpsi})})]|_{\bfpsi = 0}^\top =\\
    &\frac{\frac{\partial \cos{\delta(\bfxi_{\bfX \exp{(\hat{\bfpsi})}, \bfV})}}{\partial \bfpsi} \eta(\bfX) - (1 + \cos{\delta(\bfxi_{\bfX, \bfV})}) \frac{\partial \eta(\bfX \exp{(\hat{\bfpsi})})}{\partial \bfpsi}}{\eta^2(\bfX)}, \notag
\end{aligned}
\end{equation} 
where $\eta(\bfX)$ is defined in \eqref{eq:approx_smi_grad}. Both partial derivations in the numerator require computing $\frac{\partial \cos{\delta(\bfxi_{\bfX \exp{(\hat{\bfpsi})}, \bfV})}}{\partial \bfpsi}$. Applying the chain rule, we have:
\begin{equation}
\begin{aligned}
    &\frac{\partial \cos{\delta(\bfxi_{\bfX \exp{(\hat{\bfpsi})}, \bfV})}}{\partial \bfpsi} = \frac{\partial \cos{\delta}}{\partial \delta} \bigg|_{\delta = \delta(\bfxi_{\bfX, \bfV})}\\
    & \times \frac{\partial \delta(\bfxi)}{\partial \bfxi} \bigg |_{\bfxi = \bfxi_{\bfX, \bfV}} \frac{\partial \log(\exp(-\hat{\bfpsi}) \bfX^{-1} \bfV)^\vee}{\partial \bfpsi} \bigg|_{\bfpsi = 0}.\notag
\end{aligned}
\label{eq:partials}
\end{equation}
The first two partial derivatives can be obtained via differentiation in $\bbR$ and $\bbR^6$, respectively. The last partial derivative can be obtained via applying small perturbation $\bfpsi$ in the robot frame:
\begin{align}
    \frac{\partial \log(\exp(-\hat{\bfpsi}) \bfX^{-1} \bfV)^\vee}{\partial \bfpsi} \bigg|_{\bfpsi = 0} = \calJ_R(- \bfxi_{\bfX, \bfV}) \calJ_R(\bfxi_{\bfX}). \notag 
\end{align}
Summing over all $\bfV \in \hat{\calG}$ leads to the expression in \eqref{eq:approx_smi_grad}.\qed

\balance
{\small
\bibliographystyle{cls/IEEEtran}
\bibliography{bib/IEEEexample.bib}

\begin{thebibliography}{10}
\providecommand{\url}[1]{#1}
\csname url@rmstyle\endcsname
\providecommand{\newblock}{\relax}
\providecommand{\bibinfo}[2]{#2}
\providecommand\BIBentrySTDinterwordspacing{\spaceskip=0pt\relax}
\providecommand\BIBentryALTinterwordstretchfactor{4}
\providecommand\BIBentryALTinterwordspacing{\spaceskip=\fontdimen2\font plus
\BIBentryALTinterwordstretchfactor\fontdimen3\font minus
  \fontdimen4\font\relax}
\providecommand\BIBforeignlanguage[2]{{%
\expandafter\ifx\csname l@#1\endcsname\relax
\typeout{** WARNING: IEEEtran.bst: No hyphenation pattern has been}%
\typeout{** loaded for the language `#1'. Using the pattern for}%
\typeout{** the default language instead.}%
\else
\language=\csname l@#1\endcsname
\fi
#2}}

\bibitem{ProbabilisticRoboticsBook}
S.~Thrun, W.~Burgard, and D.~Fox, \emph{Probabilistic Robotics}.\hskip 1em plus
  0.5em minus 0.4em\relax MIT Press Cambridge, 2005.

\bibitem{julian2013mutual}
B.~J. Julian, ``Mutual information-based gradient-ascent control for
  distributed robotics,'' Ph.D. dissertation, Massachusetts Institute of
  Technology, 2013.

\bibitem{maver1993occlusions}
J.~Maver and R.~Bajcsy, ``Occlusions as a guide for planning the next view,''
  \emph{IEEE transactions on pattern analysis and machine intelligence},
  vol.~15, no.~5, pp. 417--433, 1993.

\bibitem{atanasov2014nonmyopic}
N.~Atanasov, B.~Sankaran, J.~Le~Ny, G.~J. Pappas, and K.~Daniilidis,
  ``Non-myopic view planning for active object classification and pose
  estimation,'' \emph{IEEE Transactions on Robotics}, vol.~30, no.~5, pp.
  1078--1090, 2014.

\bibitem{ma2021fov}
H.~Ma, W.~Zou, S.~Sun, Z.~Zhu, and Z.~Kang, ``{FOV} constraint region analysis
  and path planning for mobile robot with observability to multiple feature
  points,'' \emph{International Journal of Control, Automation and Systems},
  vol.~19, no.~11, pp. 3785--3800, 2021.

\bibitem{icr}
S.~Koga, A.~Asgharivaskasi, and N.~Atanasov, ``{Active Exploration and Mapping
  via Iterative Covariance Regulation over Continuous SE(3) Trajectories},'' in
  \emph{IEEE/RSJ International Conference on Intelligent Robots and Systems
  (IROS)}, 2021, pp. 2735--2741.

\bibitem{rocha2005cooperative}
R.~Rocha, J.~Dias, and A.~Carvalho, ``{Cooperative multi-robot systems: A study
  of vision-based 3-d mapping using information theory},'' \emph{Robotics and
  Autonomous Systems}, vol.~53, no. 3-4, pp. 282--311, 2005.

\bibitem{charrow2015information}
B.~Charrow, G.~Kahn, S.~Patil, S.~Liu, K.~Goldberg, P.~Abbeel, N.~Michael, and
  V.~Kumar, ``Information-theoretic planning with trajectory optimization for
  dense 3d mapping.'' in \emph{Robotics: Science and Systems}, vol.~11, 2015,
  pp. 3--12.

\bibitem{csqmi}
B.~{Charrow}, S.~{Liu}, V.~{Kumar}, and N.~{Michael}, ``Information-theoretic
  mapping using {C}auchy-{S}chwarz quadratic mutual information,'' in
  \emph{{IEEE} Int. Conf. on Robotics and Automation (ICRA)}, 2015.

\bibitem{DiAIM}
D.~Deng, R.~Duan, J.~Liu, K.~Sheng, and K.~Shimada, ``Robotic exploration of
  unknown 2d environment using a frontier-based automatic-differentiable
  information gain measure,'' in \emph{IEEE/ASME International Conference on
  Advanced Intelligent Mechatronics (AIM)}, 2020, pp. 1497--1503.

\bibitem{DiAIM-3D}
D.~Deng, Z.~Xu, W.~Zhao, and K.~Shimada, ``Frontier-based
  automatic-differentiable information gain measure for robotic exploration of
  unknown 3d environments,'' \emph{arXiv preprint arXiv:2011.05288}, 2020.

\bibitem{frontier-3d-failure}
S.~Shen, N.~Michael, and V.~Kumar, ``Stochastic differential equation-based
  exploration algorithm for autonomous indoor 3d exploration with a
  micro-aerial vehicle,'' \emph{The International Journal of Robotics Research
  (IJRR)}, vol.~31, no.~12, pp. 1431--1444, 2012.

\bibitem{frontier}
B.~{Yamauchi}, ``A frontier-based approach for autonomous exploration,'' in
  \emph{IEEE International Symposium on Computational Intelligence in Robotics
  and Automation}, 1997, pp. 146--151.

\bibitem{fsmi}
Z.~{Zhang}, T.~{Henderson}, V.~{Sze}, and S.~{Karaman}, ``Fsmi: Fast
  computation of {S}hannon mutual information for information-theoretic
  mapping,'' in \emph{{IEEE} Int. Conf. on Robotics and Automation (ICRA)},
  2019, pp. 6912--6918.

\bibitem{ssmi_icra}
A.~Asgharivaskasi and N.~Atanasov, ``{Active Bayesian Multi-class Mapping from
  Range and Semantic Segmentation Observations},'' in \emph{{IEEE} Int. Conf.
  on Robotics and Automation (ICRA)}, 2021.

\bibitem{knuth2012collaborative}
J.~Knuth and P.~Barooah, ``Collaborative 3d localization of robots from
  relative pose measurements using gradient descent on manifolds,'' in
  \emph{{IEEE} Int. Conf. on Robotics and Automation (ICRA)}, 2012.

\bibitem{mizuno2019enhanced}
N.~Mizuno, K.~Ohno, R.~Hamada, H.~Kojima, J.~Fujita, H.~Amano, T.~Westfechtel,
  T.~Suzuki, and S.~Tadokoro, ``Enhanced path smoothing based on conjugate
  gradient descent for firefighting robots in petrochemical complexes,''
  \emph{Advanced Robotics}, vol.~33, no.~14, pp. 687--698, 2019.

\bibitem{julian}
B.~J. {Julian}, S.~{Karaman}, and D.~{Rus}, ``On mutual information-based
  control of range sensing robots for mapping applications,'' in
  \emph{{IEEE/RSJ} Int. Conf. on Intelligent Robots and Systems (IROS)}, 2013.

\bibitem{BarfootBook}
T.~D. Barfoot, \emph{State Estimation for Robotics}.\hskip 1em plus 0.5em minus
  0.4em\relax Cambridge University Press, 2017.

\bibitem{octomap}
A.~Hornung, K.~M. Wurm, M.~Bennewitz, C.~Stachniss, and W.~Burgard,
  ``{OctoMap}: An efficient probabilistic {3D} mapping framework based on
  octrees,'' \emph{Autonomous Robots}, 2013, software available at
  \url{https://octomap.github.io}.

\bibitem{ros}
M.~Quigley, K.~Conley, B.~Gerkey, J.~Faust, T.~Foote, J.~Leibs, R.~Wheeler,
  A.~Y. Ng, \emph{et~al.}, ``{ROS: an open-source Robot Operating System},'' in
  \emph{ICRA Workshop on Open-source Software}, 2009.

\bibitem{icp}
A.~{Censi}, ``{An ICP variant using a point-to-line metric},'' in \emph{{IEEE}
  Int. Conf. on Robotics and Automation (ICRA)}, 2008, pp. 19--25.

\end{thebibliography}
}

\end{document}